\documentclass[letterpaper, 10 pt, conference]{ieeeconf}  % Comment this line out
\IEEEoverridecommandlockouts                              % This command is only
\usepackage{amsfonts}       % blackboard math symbols
\usepackage{xcolor}         % colors
\usepackage{amsmath,amssymb,amsfonts,dsfont,bm}
\usepackage{xcolor}
\usepackage{graphicx}
\usepackage{subcaption}
\usepackage{caption}

\newtheorem{definition}{Definition}[section]

\newtheorem{theorem}{Theorem}[section]

\newtheorem{lemma}[theorem]{Lemma}
\newcommand{\Ac}{\mathcal{A}}

\newcommand{\Fc}{\mathcal{F}}

\newcommand{\Sc}{\mathcal{S}}

\newcommand{\av}{\boldsymbol{a}}

\newcommand{\norm}[1]{{\| #1\|}}

\newcommand{\change}[1]{{\color{black}#1}} 

\usepackage{graphics} % for pdf, bitmapped graphics files
\usepackage{amsmath} % assumes amsmath package installed
\usepackage{amssymb}  % assumes amsmath package installed
\usepackage{hyperref}

\usepackage[normalem]{ulem}
\usepackage{xcolor}

\title{\LARGE \bf
Linear Convergence of Independent Natural Policy Gradient in Games with Entropy Regularization
}

\author{Youbang Sun, Tao Liu, P. R. Kumar and Shahin Shahrampour% <-this % stops a space
\thanks{Youbang Sun and Shahin Shahrampour are with Department of Mechanical \& Industrial Engineering,
        Northeastern University, Boston, MA 02115, USA
        {\tt\small \{sun.youb \& s.shahrampour\}@northeastern.edu}}%
\thanks{Tao Liu and P. R. Kumar are with the Department of Electrical \& Computer Engineering, Texas A\&M University,
        College Station, 77843, USA
        {\tt\small \{tliu \& prk\}@tamu.edu}}%
}

\begin{document}

% \IEEEpeerreviewmaketitle

\maketitle
\thispagestyle{empty}
\pagestyle{empty}

%%%%%%%%%%%%%%%%%%%%%%%%%%%%%%%%%%%%%%%%%%%%%%%%%%%%%%%%%%%%%%%%%%%%%%%%%%%%%%%%
\begin{abstract}
This work focuses on the entropy-regularized independent natural policy gradient (NPG) algorithm in multi-agent reinforcement learning.
In this work, agents are assumed to have access to an oracle with exact policy evaluation and seek to maximize their respective independent rewards. 
Each individual's reward is assumed to depend on the actions of all the agents in the multi-agent system, leading to a game between agents. 
\change{We assume all agents make decisions under a policy with bounded rationality, which is enforced by the introduction of entropy regularization.
In practice, a smaller regularization implies the agents are more rational and behave closer to Nash policies.
On the other hand, agents with larger regularization acts more randomly, which ensures more exploration.}
We show that, under sufficient entropy regularization, the dynamics of this system converge at a linear rate to the quantal response equilibrium (QRE). 
Although regularization assumptions prevent the QRE from approximating a Nash equilibrium, our findings apply to a wide range of games, including cooperative, potential, and two-player matrix games. We also provide extensive empirical results on multiple games (including Markov games) as a verification of our theoretical analysis.

\end{abstract}

%%%%%%%%%%%%%%%%%%%%%%%%%%%%%%%%%%%%%%%%%%%%%%%%%%%%%%%%%%%%%%%%%%%%%%%%%%%%%%%%
\section{INTRODUCTION}

In the emerging field of reinforcement learning (RL), the topic of multi-agent reinforcement learning (MARL) has been increasingly gaining attention.
This surge in interest may be attributed to the fact that many real-world problems are multi-agent in nature, including tasks such as
robotics \cite{sartoretti2019distributed}, modern production systems \cite{bakakeu2021multi}, economic decision making \cite{trott2021building}, and autonomous driving \cite{shalev2016safe}. 

% While it seems straightforward to apply single-agent RL algorithms such as policy gradient (PG) and natural policy gradient (NPG) directly to each agent in MARL, there are many challenges in the analysis of multi-agent systems.
Although applying single-agent RL algorithms, like policy gradient (PG) and natural policy gradient (NPG), to individual agents in MARL may seem straightforward, analyzing multi-agent systems presents numerous challenges.
In the single-agent setting, the optimal policy selects the action with the highest cumulative reward and converges to the unique global optimal solution. However, in the multi-agent setting, the global policy is constructed by taking the product of the local policies. Agents have individual rewards in general, but each individual reward depends on the global actions of all, leading to a game between agents. Even for a game as simple as a two-agent cooperative matrix game, there can be potentially multiple local stationary points. 
These stationary points are known as Nash \change{E}quilibria (NE), where no agent can enjoy a larger reward by unilaterally changing its strategy.
% all agents are inclined to keep their current policy if all other agents do not change their policies.

For general games, it is known that a system where each agent follows the policy gradient update (i.e., gradient play) can easily fail \cite{shapley1964some}.
% For general games, it is known that a gradient play, \ybsun{i.e. each agent follows the policy gradient update} \prk{say what it is} on the local agents can easily fail \cite{shapley1964some}.
For a game to converge to an NE through gradient play, additional assumptions are needed, such as the assumption of a potential function and isolatedness of the Nash equilibria \cite{sun2023provably}. In addition, in MARL we encounter similar challenges as in single-agent RL, 
including navigating sub-optimal regions characterized by flat curvatures and managing the exploration-exploitation trade-off.
% such as sub-optimal regions with flat curvatures and the exploration-exploitation trade-off. 
One mitigation strategy in practice is to enforce an entropy regularization \cite{mei2020global, kim2023adaptive, lan2023policy}.

Intuitively speaking, the addition of an entropy regularization term penalizes the policies that are not stochastic enough. The entropy regularization \change{places rationality into agents, where decisions are selected to be satisfactory rather than optimal, this} encourages the exploration of agents and prevents the system from being stuck at local sub-optimal policies \change{caused by pure strategies. The introduction of entropy was also highlighted by Soft Actor Critic \cite{haarnoja2018soft}, which is widely used today in Robotics.}
% \change{Entropy in games is also connected to}
When entropy is introduced into the problem, the system converges to the quantal response equilibrium (QRE) \cite{mckelvey1995quantal} instead of NE. \change{A QRE refers to an equilibrium with bounded rationality, which we formally define in Def. \ref{def_QRE}.}

In this paper, we consider a general static game, where the system state is assumed to be fixed and no additional assumptions on rewards for all agents are imposed. Our framework subsumes various settings, such as cooperative games, potential games, and two-player matrix games. 
% \ybsun{maybe move to related works.}
% We specifically consider the system to be static, where the system state is assumed to be fixed. \prk{This seems to contradict the earlier statement that we consider a ``general game setting."}

\subsection{Contributions} \label{sec:contribution}
Motivated by the effectiveness of entropy regularization in both single-agent RL and certain multi-agent settings \change{in games}, we have adapted the entropy-regularized natural policy gradient algorithm to games.
\change{While some existing works like \cite{cen2022independent} use QRE to approximate NE for some structured games, we consider the regularization as a given factor, and study the convergence for general games. We summarize our contribution as follows.
\begin{enumerate}
    \item We consider the NPG update with entropy regularization in games, and provide the exact algorithm update in Section \ref{sec:update}
    \item We study the convergence properties of the proposed algorithm and show in Section \ref{sec:convergence} that the system can reach a QRE at a linear convergence rate when entropy regularization is large enough.
    \item In Section \ref{sec:exps}, we present extensive numerical experiments demonstrating the effectiveness of the algorithm, and provide some discussion on its performance across various settings.
\end{enumerate}
}
% In Section \ref{sec:thm}, we study the convergence properties of the proposed algorithm and show that the system can reach a QRE at a linear convergence rate when entropy regularization is large enough. %{\rd We also provide a discussion on the key factors that affect the convergence speed of the system, which also determines the factor required for regularization.}
 %The experimental results also have close alignments with our analysis.
 
Although our theoretical analysis only considers the static game setting, in Section \ref{sec:markov-exps} we conduct experiments for stochastic (Markov) games. We show that similar empirical results also hold for Markov games, the theoretical investigation of which is useful for future work.

\subsection{Related Works}
This section offers a review of the related literature on the topics of policy gradient-based algorithms in RL and independent learning in games.

\paragraph{Policy Gradient}
There is a lot of interest in the theoretical understanding of policy gradient methods in recent literature \cite{agarwal2021theory, bhandari2019global}. There are many variations of policy gradient methods under different parameterizations. An important extension of the policy gradient method is the natural policy gradient (NPG) method \cite{agarwal2021theory, liu2020improved, muller2024geometry}, which introduces the addition of pre-conditioning in the policy update based on the problem's geometry.
% 
% However, in specific problems, these algorithms still suffer from potential slow convergence when there exists a plateau comprised of sub-optimal policies. 
To promote exploration and improve stochasticity within the system, entropy regularization has been introduced. In general, entropy regularization has been shown to accelerate convergence rates for several algorithms. 
Policy gradient methods with entropy regularization include \cite{mei2020global} for PG and \cite{cen2020fast} for NPG. 
% These regularized methods are often also referred to as the policy mirror descent method, and 
Additionally, a broad class of convex regularizers has been proposed in \cite{lan2023policy, zhan2023policy}. 

% Lastly, there is a recent line of work discussing the existence of local linear convergence of policy gradient methods.
% % , without the addition of entropy or strongly convex regularization. 
% However, these types of methods rely on factors such as adaptive step-size \cite{khodadadian2022linear}, exact line search \cite{bhandari2021linear}, additional gradient-domination-type assumptions, and geometrically increasing step-sizes \cite{xiao2022convergence}.

\paragraph{Independent Learning in Games}
Recent years have witnessed significant progress in understanding the system dynamics of independent learning algorithms in games. 
It has been shown in game theory that a system where agents use simple gradient play in a game could fail to converge, such as the ``cycling problem" shown in \cite{schafer2019competitive}.
Therefore, additional settings, namely a competitive setting and a cooperative setting have been considered.
For the competitive setting, zero-sum games have been studied by \cite{daskalakis2020independent, wei2021last}. 
% There are also studies regarding the more general setting, such as general-sum games \cite{hambly2023policy}.
A framework more general than the cooperative setting is the potential game setting \cite{cen2022independent}, where agents do not have the same rewards, but there exists a potential function tracking the value changes across all agents. 

These settings have also been extended from static games to stochastic games, where the system follows a Markov state transition model. A series of works tackle the system convergence rate in the Markov potential games setting \cite{zhang2021gradient, zhangglobal,sun2023provably}.

Entropy regularization is also widely considered in games. The convergence rate has been shown to be linear for two-player zero-sum games \cite{cen2021fast, cen2022faster}, and sub-linear for potential games \cite{cen2022independent}. \change{In particular, \cite{cen2022independent} studies NPG for potential games with entropy regularization and is of great relevance to our work.}

We note that although there are many works consider entropy regularization in games and study the system convergence to QRE, \change{most of the previous works, including \cite{cen2022independent},  address arbitrarily small regularization }\change{factors} and view QRE as an approximation of NE. Although these works provide theoretical insights that effectively demonstrate the intended use case, their effectiveness is largely confined to games with structure, such as zero-sum games or potential games. These works do not consider more general games.
% they do not address the problem that only games with structure, such as zero-sum games or potential games can be solved effectively. 
In contrast, our work considers regularization as a constant penalizing factor and discusses the convergence of the system dynamics for regularized system rewards.

\section{\MakeUppercase{Problem Formulation}}
In this section, we introduce the basic setting for a general multi-player game with the consideration of entropy regularization.

\change{

Throughout this paper, we use $\norm{a}_1$ to denote the $L^1$ norm of $a$ and $\norm{a}_\infty$ for the $L^{\infty}$ norm. We denote $[n] := \{1, ..., n\}$. For the time varying sequence of a set of parameters $\{\theta^k_i\}_{i \in [n], k \in \mathbf{N}}$, we use superscript $\theta^k$ to denote the $k$-th time step and subscript $\theta_i$ to denote the $i$-th item in the set.
% \subsection{Notations}

}

\subsection{Multi-Agent Games}
Consider a \change{tabular} strategic game $\mathcal{G} = (n, \mathcal{A}, \{r_i\}_{i \in [n]})$ consisting of $n$ agents.
The global \change{discrete} action space $\mathcal{A}= \mathcal{A}_1 \times ... \times \mathcal{A}_n$ is the product of individual action spaces, with the global action denoted by $\av := (a_1,..., a_n)$. The reward for each agent $i\in [n]$ is denoted as $r_i:  \Ac \to [0, 1]$. 

A mixed strategy for the entire system is a \textit{decentralized} multi-agent policy \cite{zhangglobal}, where all agents make decisions independently.
Therefore, the global system policy is denoted by $\pi\in \Delta(\Ac_1) \times \dots \times \Delta(\Ac_n) \subset \Delta(\Ac)$, where $\Delta$ denotes the probability simplex operator. We can write $\pi(\av) = \prod_{i\in [n]} \pi_i(a_i)$, where $\pi_i \in \Delta(\mathcal{A}_i)$ is the local policy for agent $i$. We also denote the combined policy of all agents other than $i$ as $\pi_{-i} := \prod_{j\in [n]\backslash\{i\}} \pi_j$, so that 
$\pi = \pi_i \times \pi_{-i}$. Similarly, we denote the combined action as $\av = (a_i, a_{-i})$, where $a_{-i} := \{a_j\}_{j \neq i}$.

With a slight abuse of notation, we represent the expectation of the reward $r_i$ under policy $\pi$ as $$r_i(\pi) := \mathbb{E}_{\av \sim \pi}[r_i(\av)].$$ 

We also define the \textit{marginalized reward function} of reward $r_i$ with respect to the policy $\pi_{-i}$ as $$\change{\bar{r}_i(a_i)} := \mathbb{E}_{a_{-i} \sim \pi_{-i}}[r_i(a_i, a_{-i})].$$
\change{We note that the calculation of the marginalized reward requires $\pi_{-i}$, the current policy of all players in the network. Furthermore, $\pi_{-i}$ is omitted in notation when the corresponding policy is clear from context.}

Next, we introduce the notion of Nash equilibrium (NE) \cite{nash1950equilibrium} in games. 
% In our game setting, each agent seeks to maximize its own reward in expectation, i.e., solving the following problem: $\argmax_{\pi_i} r_i(\pi_i, \pi_{-i})$. The steady-state of this system is referred to as the NE of the system. 

\begin{definition}[Nash Equilibrium]
    A joint policy $\pi^*$ is a Nash equilibrium if
    \begin{align*}
        r_i(\pi_i^*, \pi_{-i}^*) \geq r_i(\pi_i', \pi_{-i}^*), \ \ \forall \pi_i' \in \Delta (\mathcal{A}_i), \forall i \in [n].
    \end{align*}
\end{definition}

% NE is a steady-state of the system in the sense that no agent is inclined to unilaterally change its respective policy when NE has been reached as this change will not increase the reward. 
% It is known that there always exists at least one NE in a game with finite number of agents and actions when agents are allowed to use randomized policies. \prk{Provide reference. Also mentioned randomized nature of policy.}
\vspace{0.2cm}
It is known that if mixed strategies (where a player assigns a strictly positive probability to every pure strategy) are allowed, at least one NE exists in any finite game \cite{nash1950equilibrium}.

\subsection{Entropy Regularization in Games}
The Shannon entropy of policy $\pi_i$, defined as
$$\mathcal{H}(\pi_i) := -\sum_{a_i \in \mathcal{A}_i} \pi_i(a_i) \log \pi_i(a_i),$$ measures the level of randomness in actions of agent $i$. 
When entropy is added to the problem, the \change{regularized} objective for agent $i$ is modified to $\change{\hat{r}_{i}}(\pi) := r_{i}(\pi)+\tau \mathcal{H}(\pi_i)$.

With the consideration of entropy, a new type of equilibrium for the system has been defined in \cite{mckelvey1995quantal}, referred to as the quantal response equilibrium (QRE) or logit equilibrium.

\begin{definition}[Quantal Response Equilibrium]\label{def_QRE}
    A joint policy $\pi^*$ is a quantal response equilibrium when it holds that for any given $\tau$,
\begin{align*}
    \change{\hat{r}_{i}}(\pi_{i}^*, \pi_{-i}^*) \geq \change{\hat{r}_{i}}(\pi_i', \pi_{-i}^*), 
    \ \ \forall \pi_i' \in \Delta (\mathcal{A}_i)&, \forall i \in [n].
\end{align*}
\end{definition}
% \ \\

\vspace{0.2cm}
It can be easily verified that when a QRE has been reached, each agent uses a policy that assigns a probability of actions according to the marginalized reward, \change{$\pi_{i}^* \propto \exp{\left(\bar{r}_{i}(\cdot)/\tau\right)}$.}

\change{This is often referred to by the literature as the policy with bounded rationality \cite{cen2022independent} with rationality parameter $\frac{1}{\tau}$. Intuitively, an NE refers to a perfectly rational policy with $\tau \rightarrow 0$, whereas a fully random policy with $\tau \rightarrow \infty$ is considered as completely non-rational.} 
% \ybsun{can we make notations simpler?}\\

% This relationship establishes a connection between the marginalized reward and the optimal policy, which is useful for the analysis in the next section.

% \prk{It may be good to explain this result for the readers.}

\section{\MakeUppercase{Main Results}}\label{sec:thm}

\change{In this section, we study the dynamics of a multi-agent system where each agent seeks a policy to maximize their individual regularized reward. We first provide the exact algorithm update in Section \ref{sec:update}, then in Section \ref{sec:convergence} we show that with under sufficient entropy regularization, the game converges to a QRE at a linear rate.}

\subsection{Algorithm Update}\label{sec:update}

\change{We first formulate the NPG update applied on agents, which has been studied for single-agent RL by \cite{agarwal2021theory}.

Since policies are constrained on the probability simplex, in order to relax this constraint, the softmax parameterization has been widely adopted. A set of unconstrained} parameters $\theta_i \in \mathbb{R}^{|\Ac_i|}$ \change{are updated}, and the policy
\change{ is calculated by
$\pi_{i}(a_j) = \frac{\exp{(\theta_i(a_j)})}{\sum_{a_k \in \Ac_i}\exp{(\theta_i(a_k))}}.$
}

In the static games setting, \change{the NPG algorithm performs gradient updates that are pre-conditioned on the problem geometry,}
\begin{equation}\label{eq:NPG}
    \theta_i^{k+1} = \theta_i^{k} + \eta \Fc_{\theta_i}^{\dagger} \frac{\partial}{\partial \theta_i}\change{\hat{r}_{i}}({\pi}),
\end{equation}
\change{where $\eta$ denoted the step-size. Based on the definition of the regularized reward $\change{\hat{r}_{i}}({\pi})$, the policy gradient for agent $i$ can be calculated as}
\begin{equation*}
        \resizebox{0.93\hsize}{!}{$
        \change{\frac{\partial \change{\hat{r}_{i}}({\pi})}{\partial \theta_{i}(a_k)} = \pi_i(a_k)\left(\bar{r}_{i}(a_k) - \tau \log (\pi_i(a_k)) - \hat{r}_{i}(\pi)\right)}. 
$}
\end{equation*}

\change{Moreover, $\Fc_{\theta_i}^{\dagger}$ in \eqref{eq:NPG} is the Moore-Penrose pseudo-inverse of the Fisher information matrix \cite{agarwal2021theory}, defined as,}
% We employ natural policy gradient updates \change{with step-size $\eta$}, where the geometry of the problem is considered in each update. The NPG algorithm works based on the Fisher information matrix,
$$\Fc_{\theta_i} := \mathbb{E}_{a_i \sim \pi_i}[(\nabla_{\theta_i} \log \pi_i(a_i))(\nabla_{\theta_i} \log \pi_i(a_i))^\top].$$

% \ybsun{explain The Fisher information matrix? }
% The preconditioned update is calculated using the Moore-Penrose pseudo-inverse of the Fisher information matrix as

\change{Using step-size $\eta$, the} corresponding update for \change{agent $i$ under} softmax parameterization \cite{cen2022independent} becomes
\begin{equation}\label{SG-entropy-update}
    \begin{aligned}
        &\pi_i^{k+1}(a_i) \propto  \pi_i^k(a_i)^{1-\eta \tau} \exp ({\eta \bar{r}_i^{k}(a_i)}),
    \end{aligned}
\end{equation}
where $\bar{r}_i^k(a_i) = \mathbb{E}_{a_{-i} \sim \pi^k_{-i}}[r_i(a_i, a_{-i})].$ Intuitively speaking, when $\pi_{-i}^k$ is fixed, \change{the game reduces to a single-agent RL problem, and }the updates shown in \eqref{SG-entropy-update} will converge to a local optimal \change{policy}, with 
\begin{equation}\label{eq:optimal_pi}
    \pi_{i}^{k\ast} (a_i) \propto \exp{(\bar{r}^{k}_i(a_i) / \tau)}, \ \ \text{for $0<\eta\tau<1$.}
\end{equation}

% \paragraph{Auxiliary Sequence}
\subsection{Convergence Analysis} \label{sec:convergence}
Before presenting our main theorem, we first introduce the notion of \textit{QRE-gap} as 
\begin{equation*}
\resizebox{.93\hsize}{!}{$
\begin{aligned}
    \textit{QRE-gap}(\pi) = \max_{i \in [n], \pi_i' \in \Delta(\Ac_i)} \{\change{\hat{r}_{i}}(\pi_i', \pi_{-i}) - \change{\hat{r}_{i}}(\pi_i, \pi_{-i})\},
\end{aligned}
$}
\end{equation*}
\change{
where the maximum is taken when $\pi_i' = \pi_{i}^{k\ast}$ given in \eqref{eq:optimal_pi}.

For the case of $\tau =0$, the agents become pure rational, and the \textit{QRE-gap} recovers \textit{NE-gap}.}
It is easily verified that a system has reached a QRE if and only if $\textit{QRE-gap} = 0$. We study the convergence of the \textit{QRE-gap} in this section. For the ease of notation, we denote the \textit{QRE-gap}  at iteration $k$ {(i.e., policy $\pi^k$)} by $\textit{QRE-gap}^k$. 

Given the definition of $\pi_{i}^{k\ast}$ in \eqref{eq:optimal_pi}, we have:
\begin{equation}\label{eq:QRE-gap}
% \resizebox{\hsize}{!}{$
\begin{aligned}
    &\textit{QRE-gap}^k =\max_{i \in [n]} \{\change{\hat{r}_{i}}(\pi_{i}^{k\ast}, \pi_{-i}^k) - \change{\hat{r}_{i}}(\pi_i^k, \pi_{-i}^k)\}. \\
    =& \max_{i\in[n]} \left[ \langle \pi_{i}^{k\ast} - \pi_{i}^{k} , \bar{r}_i^{k} \rangle + \tau (\mathcal{H}(\pi_{i}^{k\ast}) - \mathcal{H}(\pi_{i}^{k}))\right]\\
    =& \max_{i\in[n]} \big[ \langle \pi_{i}^{k\ast} - \pi_{i}^{k} , \tau\log \pi_{i}^{k\ast}\rangle\\
    & \quad \quad\quad- \tau \langle \pi_{i}^{k\ast} , \log \pi_{i}^{k\ast}\rangle + \tau \langle \pi_{i}^{k} , \log \pi_{i}^{k}\rangle
    \big]\\
    =& \tau \max_{i\in[n]} \left[  \langle \pi_{i}^{k} , \log \pi_{i}^{k}-\log \pi_{i}^{k\ast}\rangle
    \right].
\end{aligned}
% $}
\end{equation}

Motivated by \cite{cen2020fast}, we introduce the following auxiliary sequence $\{\xi^k_i \in \mathbb{R}^{|\mathcal{A}_i|}, i \in [n]\}$ \change{to help with further analysis.}
% \prk{Define what is meant by $\exp (r)$ below.}
\begin{equation}\label{eq:aux-def}
\begin{aligned}
    \xi^0_i(a_i) &= \norm{\exp (\bar{r}_i^0/\tau)}_1\pi_i^0(a_i),\\
    \xi^{k+1}_i(a_i) &= \xi^{k}_i(a_i)^{1-\eta\tau} \exp(\eta \bar{r}_i^k(a_i)).
\end{aligned}
\end{equation}

By the definition of the auxiliary sequence, two consecutive iterates $\xi^{k+1}_i(a_i)$ and $\xi^{k}_i(a_i)$ satisfy the following equality
\begin{equation}\label{eq:aux-iter}
\resizebox{.93\hsize}{!}{$
\begin{aligned}
    &\log \xi^{k+1}_i(a_i) - \bar{r}_i^{k+1}(a_i)/\tau \\
        =& (1-\eta\tau) \left(\log \xi^{k}_i(a_i) - \bar{r}_i^{k}(a_i)/\tau
        \right) + \left(\bar{r}_i^{k}(a_i) - \bar{r}_i^{k+1}(a_i)\right)/\tau.
\end{aligned}
$}
\end{equation}

It can be observed that $\pi_i^{k} \propto \xi_i^k$ according to Equations \eqref{SG-entropy-update} and \eqref{eq:aux-def}. We introduce the following lemma to establish direct relationships between $\pi_i^{k} $ and $ \xi_i^k$.

\begin{lemma}[\cite{cen2020fast}]\label{lem:pi-xi}
For any two probability distributions $\pi_1, \pi_2 \in \Delta(\Ac)$ that satisfy $$\pi_1(a) \propto \exp{(\theta_1(a))},~~\text{and}~~ \pi_2(a) \propto \exp{(\theta_2(a))},$$ 
with $\theta_1, \theta_2 \in \mathbb{R}^{|\Ac|}$, the following inequality holds
$$\norm{\log(\pi_1) - \log(\pi_2)}_\infty \leq 2 \norm{\theta_1 - \theta_2}_\infty.$$
\end{lemma}

The proof of this lemma is provided by \cite{cen2020fast}. Next, we introduce the following lemma regarding decentralized multi-agent policies.
\begin{lemma}\label{lem:diff-pi}
    For two sets of policies $\{\pi^1_i\},\{\pi^2_i\}, i \in [n]$, where each policy $\pi_i^j \in \Delta(\mathcal{A}_i)$, we have the following inequality:
    \begin{equation*}
    \resizebox{0.95\hsize}{!}{$
        \begin{aligned}
            &\sum_{a_1,...,a_n}| \pi^1_1(a_1)\times ... \times \pi^1_n(a_n)-  \pi^2_1(a_1)\times ... \times \pi^2_n(a_n)| \\
            \leq & \sum_{i\in [n]} \sum_{a_i} |\pi_i^1(a_i) - \pi^2_i(a_i)|,
        \end{aligned}
        $}
    \end{equation*}
    i.e., 
    for 
    two global policies $\mathbf{\pi}^1, \mathbf{\pi}^2$, 
    $$\|\mathbf{\pi}^1 - \mathbf{\pi}^2 \|_1 \leq \sum_i \|\pi_i^1 - \pi^2_i\|_1.$$
\end{lemma}

\change{The proof of which is in Section \ref{sec:supp}. }
From Lemma \ref{lem:diff-pi}, we are able to evaluate the difference between the marginalized rewards of two consecutive iterations, which is used to prove the theorem below:
\begin{theorem}\label{thm:static}
Consider a static game with independent NPG update shown in \eqref{SG-entropy-update}, with the regularization factor $\tau > 2\sum_{i\in [n]} |\Ac_i|$ and learning rate $\eta < \frac{1}{\tau - 2\sum_{i\in [n]} |\Ac_i|}$. We have
\begin{equation*}
\resizebox{0.93\hsize}{!}{$
    \textit{QRE-gap}^k \leq 2\tau (1-\eta \tau + 2\eta \sum_{i}|\mathcal{A}_i|)^k \norm{\log \pi^0_i - \log \pi^{0\ast}_i}_\infty,
    $}
\end{equation*}
where $\pi^{0\ast}_i \propto \exp(\Bar{r}_i^0/\tau)$.

\end{theorem}
\begin{proof}
    For agent $i$, we know that the marginalized reward is \change{is upper-bounded by $1$, with the help of Lemma \ref{lem:diff-pi}}, we have
    \begin{equation*}
    \resizebox{0.95\hsize}{!}{$
    \begin{aligned}
        \norm{\bar{r}_i^{k+1} - \bar{r}_i^k}_\infty &\leq \norm{\pi_{-i}^{k+1} - \pi_{-i}^k}_1\leq \sum_{j\neq i} \norm{\pi_{j}^{k+1} - \pi_{j}^k}_1\\
        &\leq \sum_{j} |\Ac_j| \norm{\pi_{j}^{k+1} - \pi_{j}^k}_\infty\\
        &\leq \sum_{j} |\Ac_j| \norm{\log \pi_{j}^{k+1} - \log \pi_{j}^k}_\infty\\
        &\leq 2\sum_{j} |\Ac_j| \norm{\log \xi_{j}^{k+1} - \log \xi_{j}^k}_\infty\\
        &\leq \big(2\eta \tau \sum_{j} |\Ac_j| \big)
        \max_{i \in [n]}\norm{\log \xi_{i}^{k} - \bar{r}_i^{k}/\tau}_\infty ,
    \end{aligned}
    $}
    \end{equation*}
where the second to last inequality follows from Lemma \ref{lem:pi-xi}, and the last inequality follows from the definition of $\xi_i^k.$    

\change{We can then provide an upper bound on the following term,}
\begin{equation*}
\resizebox{0.95\hsize}{!}{$
\begin{aligned}
    &\max_{i\in[n]}\norm{\log \xi^{k+1}_i - \bar{r}_i^{k+1}/\tau}_\infty \\
        \leq & \max_{i\in[n]}\{(1-\eta\tau)  \norm{\log \xi^{k}_i - \bar{r}_i^{k}/\tau }_\infty
        + \norm{\bar{r}_i^{k} - \bar{r}_i^{k+1}}_\infty/\tau\}\\
        \leq & (1-\eta\tau + 2\eta  \sum_{j} |\Ac_j|) \max_{i\in[n]}\norm{\log \xi^{k}_i - \bar{r}_i^{k} /\tau }_\infty\\
        \leq & (1-\eta\tau + 2\eta  \sum_{j} |\Ac_j|)^{ k+1} \max_{i\in[n]}\norm{\log \xi^{0}_i - \bar{r}_i^{0} /\tau }_\infty,   
\end{aligned}
$}
\end{equation*}
\change{Here, the first inequality is provided given the properties of the auxiliary sequence in \eqref{eq:aux-iter}, the second inequality comes directly from the previous upper bound, the bound is finished by recursion.}

With the help of \eqref{eq:QRE-gap}, we can bound the QRE-gap by,

\begin{equation*}
\resizebox{0.95\hsize}{!}{$
\begin{aligned}
    &\textit{QRE-gap}^k \\
    =& \max_{i\in[n]} \left[  \tau \langle \pi_{i}^{k} , \log \pi_{i}^{k}-\log \pi_{i}^{k\ast}\rangle    \right]\\
    \leq & \tau \max_{i\in[n]}\norm{\log \pi_{i}^{k} - \log \pi_{i}^{k\ast}}_\infty\\
    \leq& 2\tau \max_{i\in[n]} \norm{\log \xi^{k}_i - \bar{r}_i^{k}/\tau }_\infty\\
     \leq& 2\tau (1-\eta\tau + 2\eta \sum_{i} |\Ac_i|)^k \max_{i\in[n]}\norm{\log \xi^{0}_i - \bar{r}_i^{0}/\tau }_\infty \\
    \leq& 2\tau (1-\eta \tau + 2\eta \sum_{i}|\mathcal{A}_i|)^k \max_{i\in[n]}\norm{\log \pi^0_i - \log \pi^{0\ast}_i}_\infty .
\end{aligned}
$}
\end{equation*}

\change{We apply the Hölder's inequality in the second line, then Lemma \ref{lem:pi-xi} and the previous inequality are used to complete the proof.}

\end{proof}
\change{
Theorem \ref{thm:static} presents an interesting perspective on the choice of the regularization factor $\tau$. For a small $\tau$, the system is not guaranteed to converge. A moderate selection of $\tau$ will guarantee convergence to a QRE. As $\tau$ increases further, the rate of convergence becomes faster, yet the corresponding QRE becomes less rational and more stochastic and is generally less desirable. As shown in the next section, it is crucial to find a suitable $\tau$ that has a fast convergence speed but still retains rationality.
}
   % This command serves to balance the column lengths
% \newpage
% \subsection{Stochastic Games}

% For games with state transition, the system becomes more complicated.
\begin{figure}[!b]
    \centering
    \includegraphics[width=0.45\textwidth]{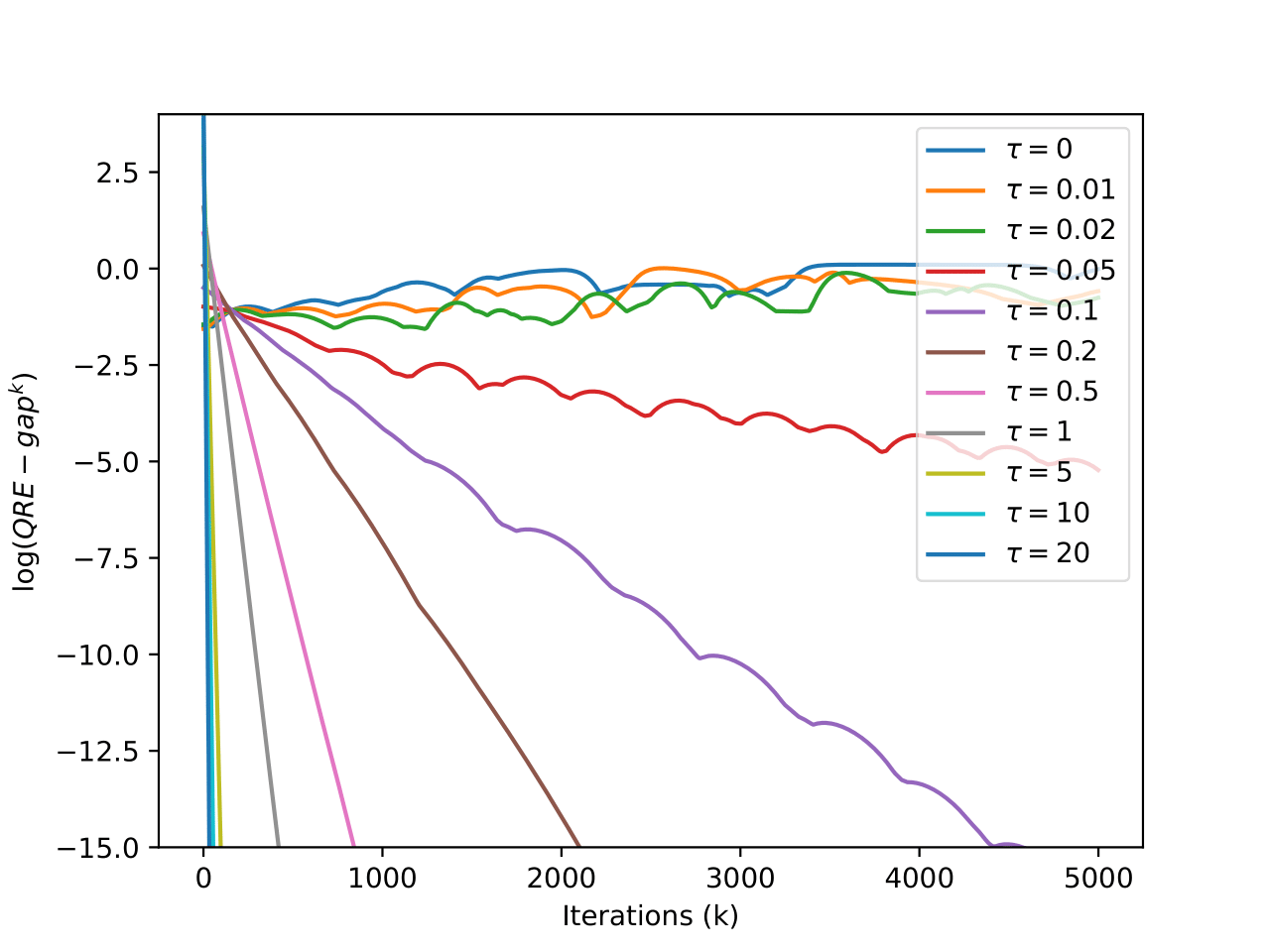}
    \caption{Convergence of \textit{QRE-gap} in random game with different $\tau$.}
    \label{fig:random-game}
\end{figure}
\section{\MakeUppercase{Numerical Results}}\label{sec:exps}

In the previous section, we established the convergence rate for \textit{QRE-gap} in games. In this section, we verify the analytical results through three sets of experiments.
\begin{figure*}
    \centering
\captionsetup{justification=centering}

\begin{subfigure}{0.32\textwidth}
    \centering
    \includegraphics[width=\textwidth]{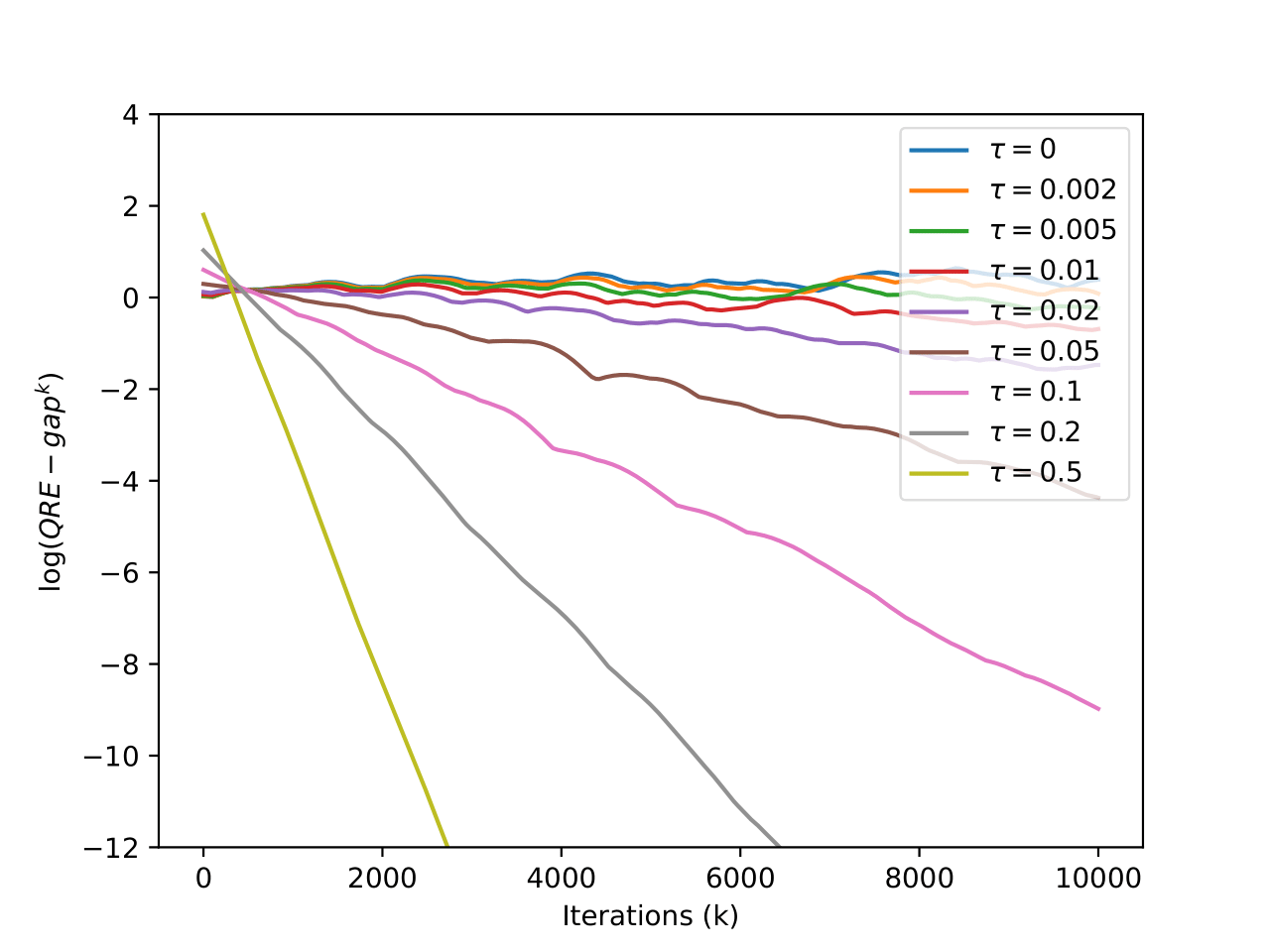}
    \caption{Convergence of \textit{QRE-gap} for \\ network zero-sum games.}
    \label{fig:ring-QRE}
\end{subfigure}
\begin{subfigure}{0.32\textwidth}
    \centering
    \includegraphics[width=\textwidth]{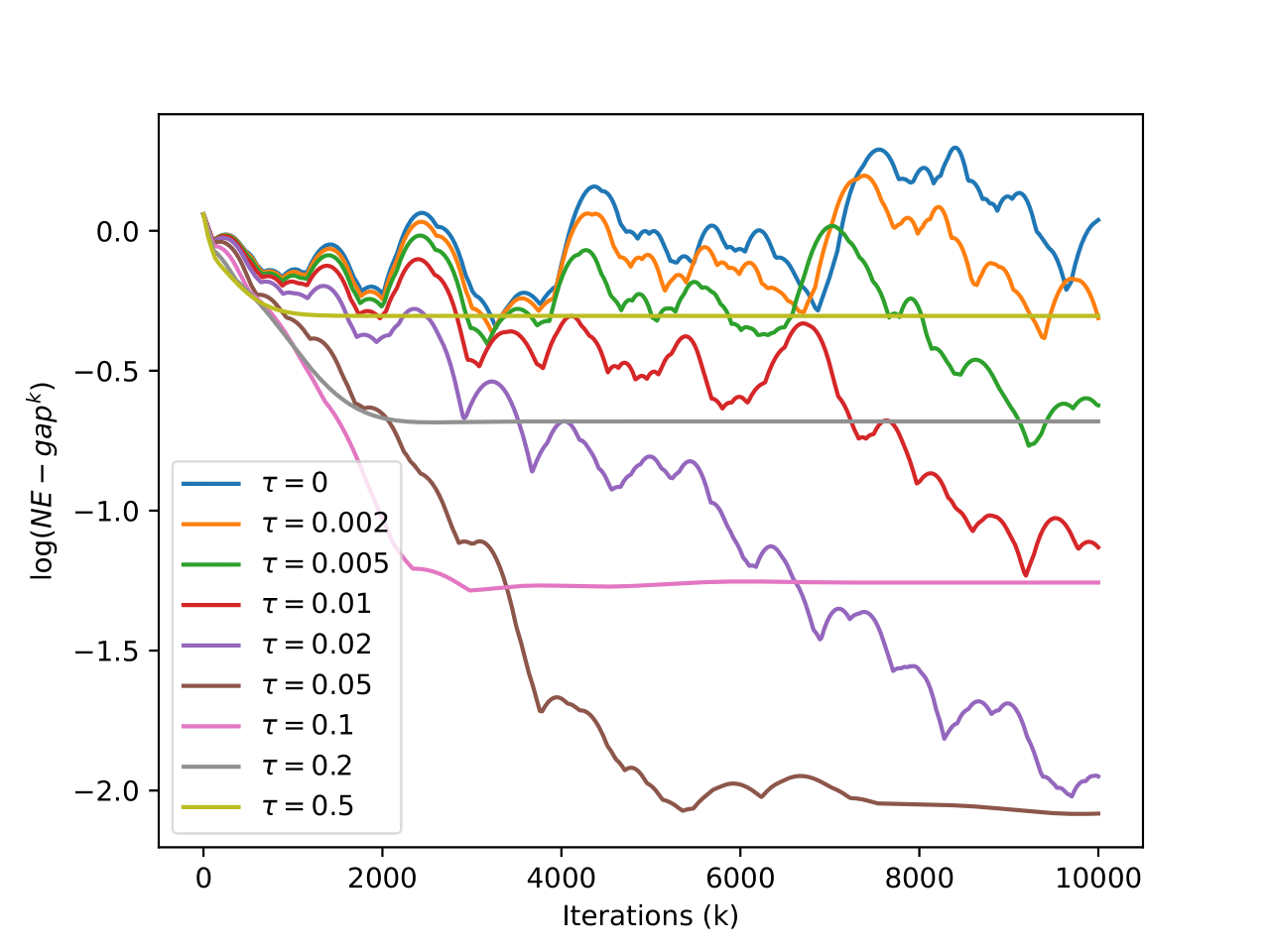}
    \caption{System \textit{NE-gap} for \\ network zero-sum games.}
    \label{fig:ring-NE}
\end{subfigure}
\begin{subfigure}{0.32\textwidth}
    \centering
    \includegraphics[width=\textwidth]{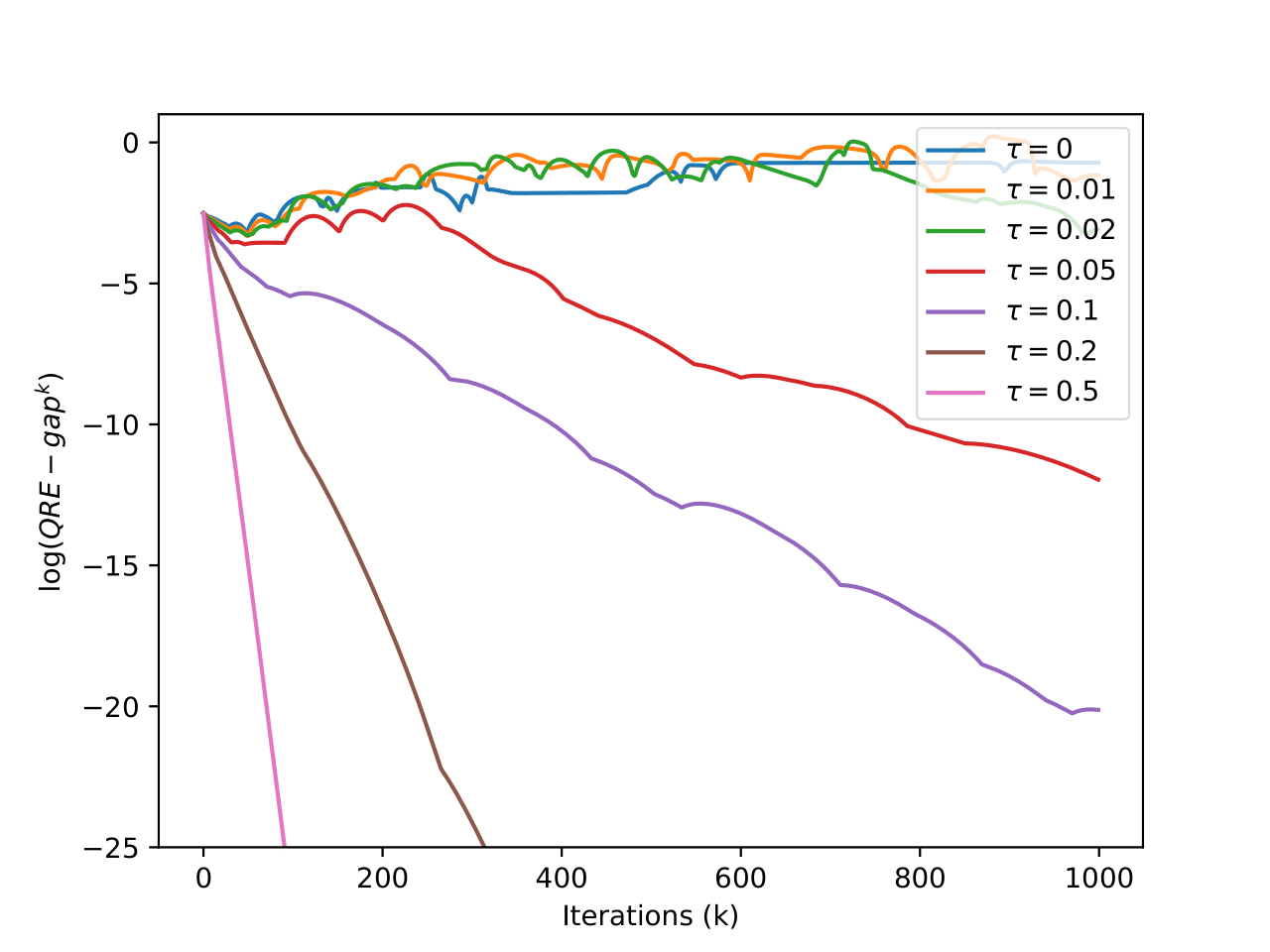}
    \caption{System \textit{QRE-gap} in synthetic Markov games with different regularization.}
    \label{fig:markov-game}
\end{subfigure}
 \caption{Network zero-sum game and Markov game}\label{fig:ring}
\end{figure*}
\subsection{Synthetic Game Setting}\label{sec:syn-game}

We first consider a multi-agent \change{system} where the rewards $r_i$ are generated randomly and independently. We set the number of agents to $n=3$, with each agent having a different \change{discrete action space size}, {$|\Ac_1| = 3, |\Ac_2| = 4, |\Ac_3| = 5$.} At the start of the experiment, all agents are initialized with random policies. 
This setting is similar to that of \cite{sun2023provably}; however, the rewards assigned to the agents in \cite{sun2023provably} satisfy the potential game assumption, but they are set to be independent in our experiments.
We use the same randomized reward and initial policy across a selection of regularization factors.
% $\tau \in \{0, 0.01, 0.02, 0.05, 0.1, 0.2, 0.5, 1\}$. 
The learning rate is set to $\eta_{\tau} = \frac{1}{2(\tau - 2\sum_{i} |\Ac_i|)}$, which is within the range given in Theorem \ref{thm:static} that guarantees convergence.
% to ensure convergence. \prk{Are we implying that there may not be convergence for other learning rates?}

The results are presented in Fig. \ref{fig:random-game} in log-scale. \change{It can be seen that when there is no regularization, or the regularization factor is negligible, the system fails to converge. As $\tau$ increases, the system converges without strict monotonicity.} When $\tau$ is large enough, the \textit{QRE-gap} decreases monotonically and converges to zero at a linear rate, with the decay rate increasing as $\tau$ gets larger. The system dynamics in Fig. \ref{fig:random-game} perfectly verifies our finding on conditions of $\tau, \eta$ and the convergence rate in Theorem \ref{thm:static}.

Analytical results and experiments both indicate that the system converges faster when there is a larger weight on regularization, with the system actually failing to converge if the regularization term is too small. 
\change{This observation aligns with our analytical results in Theorem \ref{thm:static}.}
In practice, a trade-off needs to be maintained, such that the system convergence is guaranteed, yet the QRE is still meaningful.
Furthermore, we find that for the synthetic reward experiment above, a regularization factor of $\tau \approx 0.1$ is enough for the system to converge with a linear rate. Interestingly, this requirement is significantly smaller than the requirement provided in Theorem \ref{thm:static}. This could be due to the random generation of rewards, whereas the theorem provides a bound in the worst-case scenario.

\subsection{Network Zero-Sum Games}
Next, we focus on the special game setting of zero-sum games in the network setting with polymatrix rewards \cite{bailey2019multi}. 
% In most scenarios, zero-sum games involve two agents, and the reward is defined in a matrix setting. 
This problem cannot be solved using the vanilla independent NPG update and requires additional design elements such as using extra-gradient methods \cite{cen2020fast}. However, most of the methods are restricted to the two-agent setting and cannot be adapted to the network setting.

We consider a 5-agent network with a ring graph. Each edge denotes a randomly generated zero-sum matrix game between the two neighbors. All agents are assumed to have the same action space $|\Ac_1| = ... = |\Ac_5| = 10$.

We first present the \textit{QRE-gap} in Fig. \ref{fig:ring-QRE} (linear scale). It can be seen that the convergence properties in this setting mirror those of the synthetic game in Section \ref{sec:syn-game}. We also present the results on the \textit{NE-gap} of the system shown in Fig. \ref{fig:ring-NE}. 
\change{As previously mentioned, \textit{NE-gap} can be recovered by setting $\tau=0$ in \textit{QRE-gap}. We note that \textit{NE-gap} only depends on the current policy and is independent of the algorithm updates.
The results largely mirror our first experiment,}
% From the figure we can first see that for independent NPG without regularization or with small regularization $(\tau = 0.002)$, the system still fails to converge. 
With a moderate regularization, the system is able to converge to stationarity with relatively small \textit{NE-gap}. When the regularization term is too large, the system does converge but with a somewhat undesirable \textit{NE-gap}.

\subsection{Markov Games}\label{sec:markov-exps}
%Given both the analytical and empirical convergence results for static games without the effect of state transitions, 
We now extend our experiments to the stochastic setting and use independent NPG to solve general Markov games. The Markov games setting can be seen as a generalization of the Markov decision process used in single-agent RL. Both the policy and reward depend on the current state of the system, which evolves according to a transition probability kernel $P: \Sc \times \Ac \rightarrow \Delta (\Sc)$, and the agent value function becomes a discounted cumulative reward. 

We refer to \cite{zhangglobal} for the exact problem formulation and definition for natural gradients. We define the exact update for entropy-regularized NPG in Markov games as
$$\pi_i^{k+1}(a_i|s) \propto  \pi_i^k(a_i|s)^{1-\eta \tau} \exp ({\frac{\eta}{1-\gamma} \bar{A}_i^{k}(s,a_i)}),$$
where $\bar{A}_i^{k}(s,a_i)$ denotes the marginalized advantage function defined therein.

% Analyzing Markov games is more challenging since the state space transitions add another layer of complexity. Nevertheless, we can rely on the empirical results for some insights on the dynamics of the system.

We consider a synthetic Markov game with agent number $n=3$, each with action space $|\Ac_i| = 5$, with the total number of states set to $|\Sc|=5$. 
% Similar to Section \ref{sec:syn-game}, we set the rewards for each agent to be random and independent. The transition probability kernel is also generated randomly.

We plot the log-scale results in Fig. \ref{fig:markov-game}. The figure shows that the convergence properties of Markov games closely resemble those of static games, suggesting that our theoretical results in Section \ref{sec:thm} could potentially be extended to the stochastic setting.

% \begin{figure}
%     \centering
%     \includegraphics[width=0.45\textwidth]{markov_entropy.pdf}
%     \caption{System \textit{QRE-gap} in synthetic Markov game}
%     \label{fig:markov-game}
% \end{figure}
%%%%%%%%%%%%%%%%%%%%%%%%%%%%%%%%%%%%%%%%%%%%%%%%%%%%%%%%%%%%%%%%%%%%%%%%%%%%%%%%
% \addtolength{\textheight}{-5cm}

\section{CONCLUSIONS AND FUTURE WORKS}
In this paper, we have studied the independent NPG algorithm with entropy regularization under the general multi-agent game setting. We have shown that the system converges to QRE under independent NPG updates, and that the rate of convergence is linear. However, such convergence only occurs with a sufficiently large regularization. On the other hand, a system with inadequate regularization may fail to converge. Experimental results were provided for both cases across various settings.

% \ybsun{maybe separate two parts?}
There are still many open problems that could be studied for policy gradient-based algorithms in games. A future direction of this work is to extend the analytical results to the stochastic (Markov) game setting. Our preliminary experiments have suggested that stochastic games could enjoy similar convergence to static games. This topic may contribute to multi-agent reinforcement learning, where the system is generally assumed to be Markov. 
Another potential direction is to consider the scenario where oracle information is unavailable, and the policy gradient needs to be estimated via sampling.
Lastly, our analysis can be extended to policy gradient-based algorithms such as safe MARL, robust MARL, and multi-objective MARL, following recent literature in single-agent RL \cite{zhou2023natural, zhou2022anchor}.

\section{ACKNOWLEDGMENTS}

This material is based upon work partially supported by the US Army Contracting Command under W911NF-22-1-0151 and W911NF2120064, US National Science Foundation under CMMI-2038625, and US Office of Naval Research under N00014-21-1-2385.

\bibliography{ref}
\bibliographystyle{plain}

\newpage
\onecolumn

\section*{\change{Appendix}}\label{sec:supp}
\subsection{Calculating the Policy Gradient}

We provide the exact calculations of the derivative $\frac{\partial \hat{r}_{i}({\pi})}{\partial \theta_{i}(a_k)}$ here. We note that 

First by the chain rule of composition of derivatives,
\begin{align*}
    \frac{\partial \hat{r}_{i}({\pi})}{\partial \theta_{i}} = \frac{\partial \hat{r}_{i}({\pi})}{\partial \pi_{i}} \frac{\partial \pi_{i}}{\partial \theta_{i}}
\end{align*}
where by the definition of $\hat{r}_{i}({\pi})$, we can calculate as
\begin{align*}
    \frac{\partial }{\partial \pi_{i}}\hat{r}_{i}({\pi}) &= \frac{\partial }{\partial \pi_{i}} \left(r_{i}(\pi)+\tau \mathcal{H}(\pi_i)\right)\\
    &= \frac{\partial }{\partial \pi_{i}} \left(\sum_{a_i \in \mathcal{A}_i}  \pi_i(a_i) \bar{r}_i(a_i)-\tau \sum_{a_i \in \mathcal{A}_i} \pi_i(a_i) \log \pi_i(a_i)\right)\\
    &=   \bar{r}_i(\cdot)-\tau \log \pi_i(\cdot) -\tau \mathbf{1}
\end{align*}

We next compute the Jacobian matrix  $\frac{\partial \pi_{i}}{\partial \theta_{i}}$, for the $j, k$-th element, we have

\begin{align*}
    \left[\frac{\partial \pi_{i}}{\partial \theta_{i}}\right]_{j,k} = & \frac{\partial \pi_{i}(a_j)}{\partial \theta_{i}(a_k)}\\
     = & \frac{\partial }{\partial \theta_{i}(a_k)} \frac{\exp{(\theta_i(a_j))}}{\sum_{a_l \in \mathcal{A}_i} \exp{(\theta_i(a_l))}}\\
     =& \mathbb{1}_{j=k} \frac{\exp{(\theta_i(a_j))}}{\sum_{a_l \in \mathcal{A}_i} \exp{(\theta_i(a_l))}} - \frac{\exp{(\theta_i(a_j))}\exp{(\theta_i(a_k))}}{\left(\sum_{a_l \in \mathcal{A}_i} \exp{(\theta_i(a_l))}\right)^2}\\
     = & \mathbb{1}_{j=k} \pi_i(a_j) - \pi_i(a_j)\pi_i(a_k)
\end{align*}

Therefore, we can write the derivative of 
\begin{align*}
    \frac{\partial \hat{r}_{i}({\pi})}{\partial \theta_{i} (a_k)} = & \sum_{a_j \in \mathcal{A}_i}\frac{\partial \hat{r}_{i}({\pi})}{\partial \pi_{i}(a_j)} \frac{\partial \pi_{i}(a_j)}{\partial \theta_{i}(a_k)}\\
    = & \pi_i(a_k)\left(\bar{r}_i(a_k)-\tau \log \pi_i(a_k) -\tau  \right) - \sum_{a_j \in \mathcal{A}_i} \pi_i(a_j)\pi_i(a_k) \left(\bar{r}_i(a_j)-\tau \log \pi_i(a_j) -\tau  \right) \\
    = & \pi_i(a_k)\left(\bar{r}_i(a_k)-\tau \log \pi_i(a_k) -\tau  \right) - \pi_i(a_k) \sum_{a_j \in \mathcal{A}_i} \pi_i(a_j) \left(\bar{r}_i(a_j)-\tau \log \pi_i(a_j)  \right)  + \tau \pi_i(a_k)\\
    = & \pi_i(a_k)\left(\bar{r}_i(a_k)-\tau \log \pi_i(a_k)   -  \sum_{a_j \in \mathcal{A}_i} \pi_i(a_j) (\bar{r}_i(a_j)-\tau \log \pi_i(a_j) ) \right)  \\
    = & \pi_i(a_k)\left(\bar{r}_i(a_k)-\tau \log \pi_i(a_k)   -  \hat{r}_{i}(\pi) \right)  \\
\end{align*}

\subsection{Proof of Lemma \ref{lem:diff-pi}}
In this section, we provide a detailed analysis of the previous 
% lemmas 
Lemma \ref{lem:diff-pi}
in Section \ref{sec:thm}.

\begin{proof}
    We provide the proof for $n=2$. For the case where $n>2$, we can consider the first $n-1$ probability distributions as a joint distribution, and the lemma is proven by induction. 
    % \prk{But there are multiple reward functions since the $n-1$ agents have different rewards. Does this create a problem?}

    For $n=2$, we denote the first set of policies as $\pi_1^1, \pi_2^1$, and the second set as $\pi_1^2, \pi_2^2$, respectively. We note that $\pi_i^1, \pi_i^2 \in \Delta(\Ac_i), i =  1,2.$

    The total variation distance can be bounded by the following inequality,
    \begin{align*}
        &TV(\pi_1^1\times \pi_2^1,\pi_1^2\times \pi_2^2 )\\
        = & \sum_{a_1 \in \Ac_1} \sum_{a_2 \in \Ac_2} \left|
            \pi_1^1(a_1)\pi_2^1(a_2) - \pi_1^2(a_1)\pi_2^2(a_2)
        \right|\\
        = & \sum_{a_1 \in \Ac_1} \sum_{a_2 \in \Ac_2} \big|
            \pi_1^1(a_1)\pi_2^1(a_2) - \pi_1^2(a_1)\pi_2^1(a_2) \\
            &\ \ + \pi_1^2(a_1)\pi_2^1(a_2) - \pi_1^2(a_1)\pi_2^2(a_2)
        \big|\\
        \leq & \sum_{a_1 \in \Ac_1} \sum_{a_2 \in \Ac_2}  \big|\pi_1^1(a_1) - \pi_1^2(a_1) \big|\pi_2^1(a_2) \\
        &\ \ + \pi_1^2(a_1)  \big|\pi_2^1(a_2) - \pi_2^2(a_2) \big|\\
        =& TV(\pi_1^1 ,\pi_1^2) + TV(\pi_2^1 ,\pi_2^2).
    \end{align*}
\end{proof}
% \subsection{Proof for Theorem \ref{thm:static}}
%%%%%%%%%%%%%%%%%%%%%%%%%%%%%%%%%%%%%%%%%%%%%%%%%%%%%%%%%%%%%%%%%%%%%%%%%%%%%%%%
\end{document}